\documentclass{article}

\usepackage{microtype}
\usepackage{graphicx}
\usepackage{subfigure}
\usepackage{booktabs} 

\usepackage{hyperref}

\usepackage[utf8]{inputenc} 
\usepackage[T1]{fontenc}    
\usepackage{hyperref}       
\usepackage{url}            
\usepackage{booktabs}       
\usepackage{amsfonts}       
\usepackage{nicefrac}       
\usepackage{microtype}      
\usepackage{xcolor}         
\usepackage{amsmath, amsthm} 
\usepackage{algorithm}
\usepackage{algpseudocode}
\usepackage{multirow}
\usepackage{graphicx}

\newtheorem{theorem}{Theorem}[section]     

\theoremstyle{definition}

\newtheorem{assumption}[theorem]{Assumption}

\theoremstyle{remark}
\newtheorem*{remark}{Remark}               

\usepackage[accepted]{icml2025}

\usepackage{amsmath}
\usepackage{amssymb}
\usepackage{mathtools}
\usepackage{amsthm}

\usepackage[capitalize,noabbrev]{cleveref}

\theoremstyle{plain}
\theoremstyle{definition}
\theoremstyle{remark}

\usepackage[textsize=tiny]{todonotes}

\icmltitlerunning{DP-AdamW: Investigating Decoupled Weight Decay and Bias Correction in Private Deep Learning}

\begin{document}

\twocolumn[
\icmltitle{DP-AdamW: Investigating Decoupled Weight Decay and Bias Correction in Private Deep Learning}



\icmlsetsymbol{equal}{*}

\begin{icmlauthorlist}
\icmlauthor{Jay Chooi}{equal,harv}
\icmlauthor{Kevin Cong}{equal,harv}
\icmlauthor{Russell Li}{equal,harv}
\icmlauthor{Lillian Sun}{equal,harv}

\icmlaffiliation{harv}{Harvard University, Cambridge, MA, USA}

\icmlcorrespondingauthor{Lillian Sun}{lilliansun@college.harvard.edu}
\icmlcorrespondingauthor{Kevin Cong}{kcong@college.harvard.edu}
\end{icmlauthorlist}

\icmlkeywords{privacy, differential privacy, optimizers, deep learning}

\vskip 0.3in
]



\printAffiliationsAndNotice{\icmlEqualContributionAlphabetical} 

\begin{abstract}
As deep learning methods increasingly utilize sensitive data on a widespread scale, differential privacy (DP) offers formal guarantees to protect against information leakage during model training. A significant challenge remains in implementing DP optimizers that retain strong performance while preserving privacy. Recent advances introduced ever more efficient optimizers, with AdamW being a popular choice for training deep learning models because of strong empirical performance. We study \emph{DP-AdamW} and introduce \emph{DP-AdamW-BC}, a differentially private variant of the AdamW optimizer with DP bias correction for the second moment estimator. We start by showing theoretical results for privacy and convergence guarantees of DP-AdamW and DP-AdamW-BC. Then, we empirically analyze the behavior of both optimizers across multiple privacy budgets ($\epsilon = 1, 3, 7$). We find that DP-AdamW outperforms existing state-of-the-art differentially private optimizers like DP-SGD, DP-Adam, and DP-AdamBC, scoring over 15\% higher on text classification, up to 5\% higher on image classification, and consistently 1\% higher on graph node classification. Moreover, we empirically show that incorporating bias correction in DP-AdamW (DP-AdamW-BC) consistently decreases accuracy, in contrast to the improvement of DP-AdamBC improvement over DP-Adam.
\end{abstract}

\section{Introduction}

In recent years, deep learning has achieved widespread adoption, with applications ranging from natural language processing to image generation \citep{brown_language_2020, ho_denoising_2020}. However, models trained on large and sensitive datasets have been shown to be vulnerable to privacy attacks on their training data, raising significant privacy conerns \citep{carlini2021extracting, carlini2022membership, balle2022reconstructing}. Differential Privacy (DP) provides a rigorous mathematical framework to address these concerns, enabling the development of algorithms with provable guarantees against leaking individual-specific information \citep{dwork2006calibrating}. In stochastic optimization, this has been most commonly realized by the $\textit{DP-SGD}$ algorithm introduced by \citet{abadi2016deep}. The first application of DP in the deep learning model training pipeline, this approach adds Gaussian noise of variance $\sigma^2 C^2$ after clipping per-sample gradients to a radius $C$. While DP-SGD offers strong theoretical guarantees, its empirical performance, as with vanilla SGD, often lags behind the adaptive Adam family of optimizers that dominate non-private deep learning. Subsequent work therefore lifted Adam into the private setting; specific differentially private optimizers that were developed to potentially improve the privacy-utility tradeoff include DP-GD, DP-RMSprop, and DP-Adam \citep{zhou2020private, li2022private}.

Motivated by this line of work, we study differentially private variants of the AdamW optimizer which has been empirically shown to achieve improved generalization performance over Adam, particularly enabling competitive performance with SGD on image classification tasks \cite{loshchilov2017decoupled}. To provide a differentially private implementation of AdamW, we introduce the \textit{DP-AdamW} algorithm by decoupling the weight decay from the gradient update in the DP-Adam algorithm \citep{tang_dp-adambc_2023}. Additionally, we introduce \textit{DP-AdamW-BC}, a variant of DP-AdamW that corrects for the DP bias in the second moment estimator, which \citet{tang_dp-adambc_2023} characterized and corrected for DP-Adam, resulting in DP-AdamBC.


We provide theoretical results on the privacy and convergence guarantees of the new optimizers and in particular show that DP-AdamW maintains similar privacy bounds and convergence guarantees to DP-AdamBC. We then compare the performance of our DP-AdamW and DP-AdamW-BC to previous optimizers, empirically finding that they consistently improve on the privacy-accuracy tradeoff compared to their DP-Adam counterparts. In contrast to the findings of \citet{tang_dp-adambc_2023}, we show that incorporating bias correction in DP-AdamW-BC consistently decreases accuracy across diverse tasks. We demonstrate that DP-AdamW outperforms DP-SGD even in image classification, for which Adam is known to not generalize as well \citep{loshchilov2017decoupled}.


The remainder of our paper proceeds as follows. In Section \ref{theory}, we present the DP-AdamW and DP-AdamW-BC algorithms and their theoretical guarantees. In Section \ref{exp}, we detail our experiments and present the results of using the DP-AdamW and DP-AdamW-BC optimizers in comparison to DP-Adam and DP-AdamW. In Section \ref{disc}, we discuss these results, and in Sections \ref{futurework} and \ref{concl}, we conclude and offer further avenues of work. 

\section{Motivation and Related Work}

Differentially private optimization was first studied by \cite{abadi2016deep}, who introduced Differentially Private Stochastic Gradient Descent (DP-SGD). DP-SGD achieved provable privacy guarantees via adding noise to the gradients during gradient descent. However, the additional noise led to significantly worse performance in comparison to non-private models. This led to the work of \cite{li_large_2022}, who showed that finetuning pretrained models under DP-Adam achieved strong performance on par with non-private models. Moreover, they demonstrated a scaling law of DP-models by parameters and empirically invalidated the hypothesis that DP-ML suffers from dimension-dependent performance degradation. This set the stage for practical applications of DP on deep learning. More recently, \cite{tang_dp-adambc_2023} proposed an improvement to DP-Adam, namely DP-AdamBC, adding a correction term in DP-Adam to remove the bias in the second moment estimate arising from the added gaussian noise. DP-AdamBC achieved an improvement over DP-Adam across several different target privacy budgets on image, text and graph node classification tasks. 

Motivated by these developments, our present work aims to improve the performance of DP-Adam and DP-AdamBC via adding weight decay, thereby yielding DP-AdamW and DP-AdamW-BC. We finally evaluate the performance of these optimizers using the same suite of tasks in their paper.

\section{DP-AdamW: Algorithm and Theoretical Guarantees}\label{theory}

In this section, we introduce the DP-AdamW algorithm and its variant, DP-AdamW-BC, and present theoretical results on their privacy and convergence guarantees 

\subsection{DP-AdamW: The Algorithm}\label{alg}

We now introduce the DP-AdamW algorithm. The main point is to add Gaussian noise to the gradients in the AdamW algorithm. Given a batch of gradients $g_i$, we compute the average gradient $\overline g$. Define the noised gradient $$\tilde g = \frac1B \sum_i \frac{g_i}{\max(1, \frac{\|g_i\|_2}{C})} + \frac 1B \mathcal N(0, \sigma^2 C^2 \mathbf I).$$ We then use the noised gradient $\tilde g$ as the input gradient into the AdamW algorithm of \cite{loshchilov2017decoupled}.\footnote{Alternatively, one can view this as adding decoupled weight decay into the DP-Adam algorithm of \cite{tang_dp-adambc_2023}} The resulting DP-AdamW algorithm is given in Algorithm \ref{dpadamw}.

\begin{algorithm}[!ht]
  \caption{DP-AdamW}
  \label{dpadamw}
  \begin{algorithmic}[1]
    \Require total steps $T$, learning-rate schedule $\{\eta_t\}_{t=1}^T$, clip norm $C$, weight-decay $\lambda$, hyper-parameters $\alpha,\,\beta_1,\,\beta_2$, noise multiplier $\sigma$, initial parameters $\theta_0$, numerical stability constant $\epsilon_0$
    \State $m_0 \gets 0$, $v_0 \gets 0$
    \For{$t \gets 1$ \textbf{to} $T$}
      \State $g_t \gets \nabla f\!\bigl(\theta_{t-1}\bigr)$
      \State $\displaystyle
        \tilde g_t \gets 
        \frac{1}{B}\sum_i 
        \frac{g_i}{\max\left(1,\frac{\|g_i\|_2}{C}\right)}
        \;+\; \frac{1}{B}\,\mathcal N\!\bigl(0,\sigma^2 C^2 I\bigr)$
      \State $m_t \gets \beta_1 m_{t-1} + (1-\beta_1)\,\tilde g_t$
      \State $v_t \gets \beta_2 v_{t-1} + (1-\beta_2)\,\tilde g_t^{\,2}$
      \State $\hat m_t \gets \frac{m_t}{1-\beta_1^{\,t}}$
      \State $\hat v_t \gets \frac{v_t}{1-\beta_2^{\,t}}$
      \State $\displaystyle
        \theta_t \gets \theta_{t-1}
        - \eta_t \left(
            \frac{\hat m_t}{\sqrt{\hat v_t + \epsilon_0}}
            + \lambda\,\theta_{t-1}\right)$
    \EndFor
  \end{algorithmic}
\end{algorithm}

Moreover, \cite{tang_dp-adambc_2023} observe that within the Adam update, the estimate of the second moment is biased due to the use of a noised gradient. They offer a method of bias-correction, in particular by replacing the update term $\frac{\hat m_t}{\sqrt {\hat v_t}}$ with $$\frac{\hat m_t}{\sqrt{\hat v_t - \left(\frac{\sigma C}{B}\right)^2}} = \frac{\hat m_t}{\sqrt{\hat v_t - \Phi}}.$$ We also consider this modification, which is given in Algorithm \ref{dpadamw-bc}. 

\begin{algorithm}[!ht]
  \caption{DP-AdamW-BC}
  \label{dpadamw-bc}
  \begin{algorithmic}[1]
    \Require total steps $T$, learning-rate schedule $\{\eta_t\}_{t=1}^T$, clip norm $C$, weight-decay $\lambda$, hyper-parameters $\alpha,\,\beta_1,\,\beta_2$, noise multiplier $\sigma$, initial parameters $\theta_0$, numerical stability constant $\gamma$
    \State $m_0 \gets 0$, $v_0 \gets 0$
    \For{$t \gets 1$ \textbf{to} $T$}
      \State $g_t \gets \nabla f\!\bigl(\theta_{t-1}\bigr)$
      \State $\displaystyle
        \tilde g_t \gets 
        \frac{1}{B}\sum_i 
        \frac{g_i}{\max\left(1,\frac{\|g_i\|_2}{C}\right)}
        \;+\; \frac{1}{B}\,\mathcal N\!\bigl(0,\sigma^2 C^2 I\bigr)$
      \State $m_t \gets \beta_1 m_{t-1} + (1-\beta_1)\,\tilde g_t$
      \State $v_t \gets \beta_2 v_{t-1} + (1-\beta_2)\,\tilde g_t^{\,2}$
      \State $\hat m_t \gets \frac{m_t}{1-\beta_1^{\,t}}$ 
      \State $\hat v_t \gets \frac{v_t}{1-\beta_2^{\,t}}$
      \State $\displaystyle
        \theta_t \gets \theta_{t-1}
        - \eta_t (
            \frac{\hat m_t}{\sqrt{\max(\hat v_t  - (\frac{\sigma C}{B})^2, \gamma)}}
            + \lambda\,\theta_{t-1})$
    \EndFor
  \end{algorithmic}
\end{algorithm}

\subsection{Privacy Guarantees}
Both algorithms carry the same privacy guarantees as DP-SGD and DP-AdamBC. Formally, we have the following result, analogous to the privacy guarantees in \citet{abadi2016deep} and \cite{tang_dp-adambc_2023}.

\begin{theorem}[cf. Proposition 1 of \cite{tang_dp-adambc_2023}]\label{privacy_guarantee}
Suppose that the DP-SGD optimizer $DP-SGD(\theta, X, y, C, \sigma, B)$ satisfies $(\epsilon, \delta)$-DP with privacy analysis $\phi(T, \theta_i)$. Then both $DP-AdamW(\theta, X, y, C, \sigma, B)$ and $DP-AdamW-BC(\theta, X, y, C, \sigma, B)$ satisfy $(\epsilon, \delta)$-DP with the same privacy analysis $\phi(T, \theta_i)$. 
\end{theorem}

The proof of this theorem is given in Appendix \ref{privacy_appendix_a}; it follows the outline of Proposition 1 of \cite{tang_dp-adambc_2023}.

\subsection{Convergence Guarantees}

We now show that the DP-AdamW and DP-AdamW-BC optimizers converge\footnote{Here, convergence is in the sense of `average gradient of the true objective function is small'. See Theorems \ref{conv1} and \ref{conv2} for details.} under reasonable assumptions. Our results are based on and take similar form to those of \cite{défossez2022simpleconvergenceproofadam}, who proved convergence guarantees for vanilla Adam, and \cite{tang_dp-adambc_2023}, who proved analogous guarantees for DP-AdamBC. Our main contribution is to address the weight decay term; for this, the main idea is roughly to show that the parameters $\|\theta_t\|$ are bounded and apply standard inequalities. 

Throughout this subsection, we will work under the following natural assumptions, found in \cite{tang_dp-adambc_2023}. Let $F: \mathbb R^d \rightarrow \mathbb R$ denote the objective function and $f: \mathbb R^d \rightarrow \mathbb R$ denote a stochastic function with $\mathbb E(\nabla f(\theta)) =  \nabla F(\theta)$.\footnote{That is, $F$ is the true objective function, while $f$ is the estimate of the objective function. For instance, if one is using a squared-error loss, then $f$ is the empirical MSE, while $F$ is the expected squared-error loss.} Let $\|\cdot \|$ denote the $L_2$-norm. We then have the following.

\begin{assumption}\label{as1}
    $F$ is bounded below: $F(\theta) \geq F_*$ for all $\theta$.
\end{assumption}

\begin{assumption}\label{as2}
    The gradients $\|\nabla f_t(\theta)\| \leq C_1 \leq C$ are uniformly almost surely bounded. 
\end{assumption}

\begin{assumption}\label{as3}
    The gradient of $F$ is $L$-Lipschitz continuous: $\|\nabla F(\theta) - \nabla F(\theta')\| \leq L \|\theta - \theta'\|$. 
\end{assumption}

\begin{remark}
    These are the same assumptions as found in \cite{défossez2022simpleconvergenceproofadam} and \cite{tang_dp-adambc_2023}, but we will provide some further intuition. Assumption \ref{as1} says that $F$ \textit{can} be optimized; this is clearly necessary and moreover holds in practice (for instance, $F_* = 0$ holds for any squared-error loss). Assumption \ref{as2} assumes that the gradients do not blow up. This assumption is necessary for theoretical results, but does not always hold in practice; exploding gradients often occur in empirical studies. Note that we additionally assume $C_1 \leq C$ for simplicity, since bounded gradients implies that gradient clipping is not necessary. Lastly, Assumption \ref{as3} is again necessary to attain effective theoretical bounds. It is generally true if $F$ is a reasonably `smooth` function. 
\end{remark}
We divide our results into two settings. We first have the following guarantees on DP-AdamW and DP-AdamW-BC without momentum, i.e. the regime of $\beta_1 = 0$. 

\begin{theorem}\label{conv1}
    Under Assumptions \ref{as1}, \ref{as2}, and \ref{as3}, suppose that $\beta_1 = 0$, $0 < \beta_2 < 1$, $\alpha \in (0,1)$, and the learning rate follows $\eta_t = \eta \sqrt{\frac{1 - \beta_2^t}{1 - \beta_2}}$. Let $\Phi  = \left(\frac{\sigma C}{B}\right)^2$ denote the bias correction term and let $\mu^* = \frac{\beta_2(1 - \beta_2^T)}{1 - \beta_2}[(\Phi - \tfrac{2\Phi}{\pi}) + \bigl(C + \sqrt{\tfrac{2\Phi}{\pi}}\bigr)^2]$, $\nu^* = 2\beta_2^2 \Phi \sqrt{\frac{1 - \beta_2^{2T}}{1 - \beta_2^2}}$, and $b^* = 4\beta_2 \Phi$ be constants. Then there exists a constant $c(\lambda) = c(\beta_1, \beta_2, \lambda, \eta, C_1, L, \theta_0, \epsilon_0, \Phi)$ such that whenever $$\delta_0 \geq \begin{cases}
        \mu^* + \sqrt{\ln (1 / \frac{\alpha}{2T})(2(\nu^*)^2)} & 0 \leq \delta_0 \leq \frac{(\nu^*)^2}{b^*} \\ \mu^* + \ln (1 / \frac{\alpha}{2T}) 2b^* & \delta_0 \geq \frac{(\nu^*)^2}{b^*},
    \end{cases}$$ we have with probability at least $1 - \alpha$ that for DP-AdamW, 

    \begin{center}
        $\frac1T \sum_{0}^{T-1} \mathbb E\|\nabla F(\theta_{t-1})\|^2 \leq \dfrac{2(\delta_0+C_1)(F(\theta_0) - F_*)}{\eta T}$ \\ $+ \left(\dfrac{4d(C^2 + \Phi)}{\sqrt {1 - \beta_2}} + \dfrac{\eta dL \sqrt{C^2 + \Phi}(1 + \lambda)}{1 - \beta_2}\right) \cdot \dfrac RT$ \\ $+ \dfrac{1}{2T} \left(C_1^2 + \|\theta_0\|^2 + c(\lambda) \max_t (\eta_t+\eta_t^2) R\right)$ \\ $\left(\lambda \sum_{t=1}^T \eta_t + \dfrac{L}{2} (\lambda + \lambda^2) \sum_{t=1}^T \eta_t^2\right),$
    \end{center}
    
    where $R = d\left(\ln \left(1 + \frac{C^2 + \Phi}{(1-\beta_2)\epsilon_0}\right) - T\ln \beta_2\right),$
    
    and for DP-AdamW-BC, we have $\dfrac 1T \sum_{i=0}^{T-1}\mathbb E[\|\nabla F(\theta_i)\|^2] $
    
    \begin{center}
        $\leq \dfrac{2\sqrt{(\delta_0+C_1)^2 - \Phi} (F(\theta_0) - F_*)}{\eta T}$ \\ $+ \left(\dfrac{4dC^2}{\sqrt{1 - \beta_2}} + \dfrac{\eta dL(1+\lambda)C}{1 - \beta_2}\right) \cdot \dfrac{R_{BC}}T$ \\ $+ \dfrac{1}{2T}\left(C_1^2 + \|\theta_0\|^2 + c(\lambda) \max_t (\eta_t + \eta_t^2) R_{BC}\right)$ \\ $\left(\lambda \sum_{t=1}^{T} \eta_t + \dfrac L2 (\lambda + \lambda^2) \sum_{t = 1}^T \eta_t^2\right), $
    \end{center}
    
    where $R_{BC} = d\left(\ln \left|1 - \frac{C^2 + \Phi}{(1-\beta_2)\Phi}\right| - T\ln \beta_2\right).$ 
\end{theorem}

\begin{remark}

We first make a few notes about the statement of Theorem \ref{conv1} below.

First, the definition of $\delta_0$ arises from an application of concentration bounds used in the proof; in particular, it arises in Corollary 1 of \cite{tang_dp-adambc_2023}. Essentially, the constraint on $\delta_0$ is that either $\delta_0 \geq \max (\mu^* + \ln (1 / \frac{\alpha}{2T} ) 2b^*, \frac{(\nu^*)^2}{b^*})$ or $\delta_0 \in [\mu^* + \sqrt{\ln (1 / \frac{\alpha}{2T})(2 (\nu^*)^2)}, \frac{(\nu^*)^2}{b^*}]$.

Second, the right hand side is a fairly ugly expression, but we will derive an asymptotic bound when $T \rightarrow \infty$. Supposing that $\eta = T^{-a}$ and $\beta_2 = 1 -  T^{-b}$, note that $R = \mathcal O(\ln T + T^{1-b}) = \mathcal O(T^{1-b} \ln T)$. Moreover, since $b^* = \mathcal O(1)$, $\nu^* = \mathcal O(1)$ if $b < 1$ and $\nu^* = \mathcal O(\sqrt T)$ otherwise, and $\mu^* = \mathcal O(1)$ if $b < 1$ and $\mu^* = \mathcal O(T)$ otherwise, we find that $\delta_0 = \mathcal O(\ln T)$ if $b < 1$ and $\delta_0 = \mathcal O(T)$ otherwise. Taking $b < 1$, the first inequality in the theorem becomes 
\begin{center}
    $\frac1T \sum_{i=0}^{T-1} \mathbb E[\|\nabla F(\theta_i)\|^2] \leq \ln T \cdot \mathcal O(T^{a-1}$ \\ $+ (T^{\frac b2} + T^{b-a}) T^{-b}+  T^{b/2 - a} (T^{\frac 32 - a} + T^{2-2a})T^{-b}).$
\end{center}
The right hand side is optimized for the choice $b \rightarrow 1^-$, $a = \frac23$, under which we obtain $$\frac1T \sum_{i=0}^{T-1} \mathbb E[\|\nabla F(\theta_i)\|^2]  = \mathcal O(T^{-\frac13 + \zeta} \ln T) \overset{T \rightarrow \infty}\longrightarrow 0,$$ where $\zeta$ is arbitrarily small. An analogous result\footnote{In fact, with exactly the same bound.} can be attained for DP-AdamW-BC. Hence, one can interpret these results as saying that under an `optimal` learning rate regime, the average gradient of the objective function throughout the course of DP-AdamW and DP-AdamW-BC converges to 0 at an inverse-polynomial rate as the time horizon increases. By choosing $\alpha \approx 0$, this thus implies that DP-AdamW and DP-AdamW-BC converges with high probability to a local minimum under the given conditions and assumptions.  

Lastly, we now make a few notes about the proof of Theorem \ref{conv1}. For a broad sketch, we follow the rough outline of \cite{tang_dp-adambc_2023}; this suffices to account for all terms other than the weight decay term. To address this term, we  bound the expected magnitude of the parameters, $\mathbb E \|\theta_i\|^2$, using intermediary inequalities from \cite{tang_dp-adambc_2023}. This method can be imitated to show convergence in certain cases of non-private AdamW as well, which may be of independent interest.\footnote{To our knowledge, there is no stated such result, but one can attain such a result via the same proof method.} 
\end{remark}

\medskip

We now turn to the general case where $\beta_1$ is not necessarily equal to 0. In this case, we have the following result, attained via techniques similar to those used to prove Theorem \ref{conv1}.

\begin{theorem}\label{conv2}
Under Assumptions \ref{as1}, \ref{as2}, and \ref{as3}, suppose that $0 < \beta_1 < 1$, $0 < \beta_2 < 1$, $\alpha \in (0,1)$, and the learning rate follows $\eta_t = \eta(1-\beta_1) \sqrt{\frac{1 - \beta_2^t}{1 - \beta_2}}$. Let $\Phi  = \left(\frac{\sigma C}{B}\right)^2$ denote the bias correction term and let $\mu^* = \frac{\beta_2(1 - \beta_2^T)}{1 - \beta_2}[(\Phi - \tfrac{2\Phi}{\pi}) + \bigl(C + \sqrt{\tfrac{2\Phi}{\pi}}\bigr)^2]$, $\nu^* = 2\beta_2^2 \Phi \sqrt{\frac{1 - \beta_2^{2T}}{1 - \beta_2^2}}$, and $b^* = 4\beta_2 \Phi$ be constants. Then there exists a constant $c(\lambda) = c(\beta_1, \beta_2, \lambda, \eta, C_1, L, \theta_0, \epsilon_0, \Phi)$ such that whenever $\tilde T = T - \frac{\beta_1}{1 - \beta_1} > 0$ and $$\delta_0 \geq \begin{cases}
        \mu^* + \sqrt{\ln (1 / \frac{\alpha}{2T})(2(\nu^*)^2)} & 0 \leq \delta_0 \leq \frac{(\nu^*)^2}{b^*} \\ \mu^* + \ln (1 / \frac{\alpha}{2T}) 2b^* & \delta_0 \geq \frac{(\nu^*)^2}{b^*},
    \end{cases}$$ we have with probability at least $1 - \alpha$ that for DP-AdamW, 
    
    \begin{center}
        $\mathbb E[\|\nabla F(\theta_\tau)\|^2] \leq \dfrac{2(\delta_0 + C_1)(F(\theta_0) - F_*)}{\eta \tilde T} +E\cdot R$ \\ $+ \dfrac{1}{2T}\left(C_1^2 + \|\theta_0\|^2 + c(\lambda) \max_t (\eta_t + \eta_t^2) R\right)$ \\ $\left(\lambda \sum_{t=1}^{T} \eta_t + \dfrac L2 (\lambda + \lambda^2) \sum_{t = 1}^T \eta_t^2\right),$
    \end{center}
    where 
    \begin{center}
        $E = \frac{\eta dL (1 - \beta_1)\delta_0}{(1 - \beta_1 / \beta_2)(1 - \beta_2)} + \frac{2\eta^2 dL^2 \beta_1}{(1 - \beta_1/\beta_2)(1 - \beta_2)^{3/2}}$ \\ $+ \frac{12 d\delta_0^2\sqrt{1-\beta_1}}{(1 - \beta_1/\beta_2)^{3/2} \sqrt{1 - \beta_2}}$ \\ and $R = d\left(\ln \left(1+ \frac{\delta_0^2}{\epsilon_0 (1 - \beta_2)}\right) - T \log \beta_2\right),$
    \end{center}

    and for DP-AdamW-BC, 
    \begin{center} 
        $\mathbb E[\|\nabla F(\theta_\tau)\|^2] \leq \dfrac{2\sqrt{(\delta_0 + C_1)^2 - \Phi}(F(\theta_0) - F_*)}{\eta \tilde T} +E_{BC} \cdot R_{BC}$ \\ $+ \dfrac{1}{2T}\left(C_1^2 + \|\theta_0\|^2 + c(\lambda) \max_t (\eta_t + \eta_t^2) R_{BC}\right)$ \\ $ \left(\lambda \sum_{t=1}^{T} \eta_t + \dfrac L2 (\lambda + \lambda^2) \sum_{t = 1}^T \eta_t^2\right), $
    \end{center} 
    where 
    \begin{center}
        $E_{BC} = \frac{\eta dL (1 - \beta_1)\sqrt{\delta_0^2 - \Phi}}{(1 - \beta_1 / \beta_2)(1 - \beta_2)} + \frac{2\eta^2 dL^2 \beta_1}{(1 - \beta_1/\beta_2)(1 - \beta_2)^{3/2}}$ \\ $ + \frac{12 d(\delta_0^2 -\Phi)\sqrt{1-\beta_1}}{(1 - \beta_1/\beta_2)^{3/2} \sqrt{1 - \beta_2}}$
        \\ and $R_{BC} = d\left(\ln \left|1- \frac{\delta_0^2}{\Phi (1 - \beta_2)}\right| - T \log \beta_2\right).$
    \end{center}

In both above inequalities, the left-hand-side expectations are with respect to sampling $\tau$ from the distribution $\mathbb P(\tau = t) \propto 1 - \beta_1^{T - t}$.
\end{theorem}
\begin{remark}
    This result is not an exact generalization of Theorem \ref{conv1}, but is fairly similar while holding in a more general setting. Moreover, the main steps of the proof are analogous. Lastly, we note that one can attain similar asymptotics for the bounds in this result as in Theorem \ref{conv1}. 
\end{remark}
Proofs of Theorem \ref{conv1} and \ref{conv2} are given in Appendix \ref{convergence_appendix_b}.

\section{Experiments}\label{exp}

We evaluate the performance of using the DP-AdamW and DP-AdamW-BC optimizers on image, text, and graph node classification tasks. We then compare our results to those of \cite{tang_dp-adambc_2023}. Experimental details are included in Appendix \ref{app:experimental_details}.

\subsection{Image Classification}\label{imc}

We first replicated the experimental setup of \citet{tang_dp-adambc_2023} for DP-Adam and DP-AdamBC on the same CNN architecture and CIFAR-10 task. Our DP-Adam and DP-Adam implementations achieve comparable test accuracies to those reported, validating our baseline implementations. This allows for a direct comparison between the performance of the original DP-Adam variants and our proposed DP-AdamW approaches under corresponding conditions.

We evaluate DP-AdamW and DP-AdamW-BC across the tuned hyperparameters for each $\epsilon \in \{1, 3, 7\}$. Test accuracies on CIFAR-10 are reported in Table \ref{tab:cifar10_comparison}, along with benchmark results reported in \citet{tang_dp-adambc_2023}. DP-AdamW outperforms DP-Adam in all privacy budgets. DP-AdamW-BC outperforms DP-AdamBC across privacy budgets except for the large $\epsilon = 7$. The most significant performance increase is in the $\epsilon = 3$ setting, where both DP-AdamW and DP-AdamW-BC outperform DP-Adam and DP-AdamBC by up to 5\%.

DP-AdamW achieves state-of-the-art accuracy among the evaluated DP-SGD and DP-Adam variants for CIFAR-10 image classification for tight to moderate privacy constraints ($\epsilon$=1,3). This result is unexpected, as image classification is a known setting where Adam optimizers' performance fall short in comparison to SGD \citep{loshchilov2017decoupled}. Our results confirm the effectiveness of decoupling weight decay, even in the differential privacy setting.

We observe that DP-AdamW consistently outperformed DP-AdamW-BC across all tested privacy budgets, achieving statistically significantly higher final test accuracy. In contrast to findings for DP-AdamBC from \citet{tang_dp-adambc_2023}, DP-AdamW-BC accuracies show that applying DP bias correction for the second moment estimator degrades test accuracy compared to the uncorrected DP-AdamW. While DP-AdamW-BC offers a theoretically appealing bias correction, in our CIFAR-10 experiments, the DP-AdamW results remain stronger across all privacy budgets. This suggests the decoupling weight decay in the AdamW optimizer induces a greater performance boost, and potentially conflicts with,  the effect from the bias correction.

\begin{table}[!ht]
\centering
\caption{Image Classification: performance comparison of DP optimizers on CIFAR-10 dataset}
\label{tab:cifar10_comparison}
\resizebox{\linewidth}{!}{%
\begin{tabular}{cccc}
\hline\hline
 & $\epsilon \approx 1$ & $\epsilon \approx 3$ & $\epsilon \approx 7$ \\
\hline
DP-SGD & $52.37$ $(0.50)$ & $57.30$ $(0.76)$ & $\mathbf{65.30}$ $\mathbf{(0.33)}$ \\
DP-Adam & $51.89$ $(0.69)$ & $54.08$ $(0.41)$ & $62.24$ $(0.10)$ \\
DP-AdamBC & $49.75$ $(0.56)$ & $54.27$ $(0.23)$ & $63.43$ $(0.43)$ \\
\hline
DP-AdamW & $\mathbf{52.59}$ $\mathbf{(0.44)}$ & $\mathbf{59.26}$ $\mathbf{(0.30)}$ & $63.25$ $(0.53)$ \\
DP-AdamW-BC & $51.43$ $(0.52)$ & $58.16$ $(0.72)$ & $62.01$ $(0.29)$ \\
\hline\hline
\end{tabular}%
}
\end{table}

\subsection{Text Classification}\label{tc}

Table~\ref{tab:qnli_comparison} reports the mean test accuracy and standard deviation over five random seeds for each optimizer across privacy budgets. Two trends emerge. First, decoupled weight decay yields a large performance improvement. Switching from DP-Adam to DP-AdamW raises accuracy by over 15 percentage points at every $\epsilon$ which aligns with prior findings in the literature that weight-decay coupling persists and is amplified in the private setting. The accuracies achieved by DP-AdamW without bias correction outperform all previously published DP optimizer fine tuning results on QNLI to our knowledge. Second, adding bias correction hurts on textual tasks. Inspecting all experiments, we observe that DP-AdamW-BC trails DP-AdamW by 2-3 percentage points and hypothesize that the added variance floor $\gamma$ interacts poorly with the small effective batch size and the already low intrinsic gradient noise of the BERT fine-tuning task. For all optimizers tested, test set performance improves as $\epsilon$ increases, which makes sense based on our knowledge of differential privacy because less noise must be added to satisfy looser privacy budgets, resulting in the gradient update possessing more signal comparatively.

\begin{table}[! ht]
\centering
\caption{Text Classification: Performance comparison of DP optimizers on QNLI dataset}
\label{tab:qnli_comparison}
\resizebox{\linewidth}{!}{%
\begin{tabular}{lccc}
\hline\hline
 & $\epsilon \approx 1$ & $\epsilon \approx 3$ & $\epsilon \approx 7$ \\
\hline
DP-SGD & $57.10$ $(1.59)$ & $58.85$ $(1.20)$ & $58.29$ $(0.92)$ \\
DP-Adam & $58.00$ $(2.05)$ & $60.72$ $(1.12)$ & $61.23$ $(1.30)$ \\
DP-AdamBC & ${58.32}$ ${(1.90)}$ & ${61.42}$ ${(0.99)}$ & ${62.83}$ ${(1.60)}$ \\
\hline
DP-AdamW & $\mathbf{77.78}$ $\mathbf{(0.13)}$ & $\mathbf{79.26}$ $\mathbf{(0.18)}$ & $\mathbf{80.01}$ $\mathbf{(0.19)}$ \\
DP-AdamW-BC   & $75.56$ $(0.17)$ & $76.65$ $(0.22)$ & $77.29$ $(0.08)$\\
\hline
\end{tabular}
}
\end{table}

\subsection{Graph Node Classification}\label{gnn}


We found that DP-AdamW and DP-AdamW-BC outperform DP-SGD and DP-Adam across all \(\epsilon\). Furthermore, both DP-AdamW and DP-AdamW-BC outperform DP-AdamBC for $\epsilon \approx 12$, while DP-AdamW outperforms DP-AdamBC for \(\epsilon\approx6\). This shows that DP-AdamW and DP-AdamW-BC, like their non-private counterparts, perform better than DP-SGD and DP-Adam, while exceeding the performance of DP-AdamBC under looser privacy budgets.

\begin{table}[ht]
\centering
\caption{Graph Node Classification: Performance comparison of DP optimizers on obgn-arxiv node classification}
\label{tab:gnn_comparison}
\resizebox{\linewidth}{!}{%
\begin{tabular}{lccc}
\hline\hline
 & $\epsilon \approx 3$ & $\epsilon \approx 6$ & $\epsilon \approx 12$ \\
\hline
DP-SGD       & $45.35$ $(1.38)$           & $49.12$ $(1.90)$           & $54.20$ $(0.62)$          \\
DP-Adam      & $46.55$ $(0.54)$           & $51.98$ $(0.48)$           & $54.02$ $(0.18)$          \\
DP-AdamBC    & $\mathbf{50.51}$ $(0.56)$  & $53.40$ $(0.28)$           & $53.81$ $(0.34)$          \\
\hline
DP-AdamW     & $50.20$ $(0.75)$           & $\mathbf{53.41}$ $(0.37)$  &$\mathbf{54.78}( 0.24)$       \\
DP-AdamW-BC   & ${48.99}$ ${(0.67)}$           & $52.89$ $(0.40)$           & $54.53$ $(0.43)$ \\
\hline\hline
\end{tabular}
}
\end{table}

\section{Discussion}\label{disc}

Our findings reveal three key insights into differentially private optimization.

\emph{First, DP-AdamW consistently outperforms DP-Adam} across all tested privacy budgets and tasks, particularly in the moderate privacy setting ($\epsilon=3$). This suggests that AdamW's core benefit—decoupling weight decay from adaptive gradient updates—translates effectively to the differentially private (DP) setting. By applying weight decay directly to parameters, DP-AdamW likely achieves more stable and effective regularization compared to DP-Adam, mitigating potential negative interactions between DP noise, the adaptive moment estimates, and the regularization term. Decoupled weight decay offers a twofold advantage: regularization strength becomes independent of instantaneous noise levels, and a mild pre-conditioning effect reduces the likelihood of gradients hitting the clipping threshold, thus preserving more signal.

\emph{Second, DP-AdamW demonstrates superior performance over DP-SGD on CIFAR-10 image classification under tighter privacy constraints} ($\epsilon=1, 3$), a noteworthy result given that non-private Adam variants often underperform SGD on such tasks \citep{loshchilov2017decoupled}. This strong performance indicates that the combination of adaptive learning rates and decoupled weight decay is particularly advantageous when dealing with the noise inherent in DP training. DP training introduces gradient clipping and additive noise; AdamW's coordinate-wise adaptive updates help counteract these distortions by allowing weights in low-variance directions to move appropriately without letting those in high-variance directions explode. This corrective effect is more pronounced with higher noise levels, explaining the significant gains at tighter privacy budgets.

\emph{Third, and most surprisingly, bias correction in DP-AdamW-BC consistently \textit{degrades} test accuracy compared to DP-AdamW} across all tasks and privacy budgets. This contrasts sharply with findings for \citet{tang_dp-adambc_2023}, where DP-AdamBC typically improves upon DP-Adam. While bias correction aims to counteract DP noise effects in the second moment estimate, our results suggest a negative interaction with AdamW's decoupled weight decay. The theoretical benefit of bias correction might be less relevant or even detrimental when decoupled weight decay is the primary regularization force. Specifically, the bias correction term can cause the denominator in the update rule (e.g., $\hat{v_t}-\Phi$) to become very small (clamping to $\gamma$), effectively freezing the adaptive schedule. With decoupled decay, this leads to large parameter steps that are not offset within Adam's update, potentially adding optimization noise and reducing accuracy.

Overall, our findings underscore \emph{DP-AdamW's promise as a robust optimizer for differentially private deep learning}, highlighting the benefits of decoupled weight decay under privacy constraints. The unexpected underperformance of DP-AdamW-BC warrants further investigation into the complex interplay between DP noise, adaptive moment estimation, bias correction, and regularization strategies.

\section{Conclusion}\label{concl}

This paper analyzes DP-AdamW and DP-AdamW-BC, differentially private versions of the AdamW optimizer, establishing their theoretical privacy and convergence guarantees which align with existing DP-SGD and DP-Adam literature. Empirical evaluations on image, text, and graph node classification tasks reveal that DP-AdamW generally outperforms standard baselines like DP-SGD and DP-Adam, especially under tighter privacy constraints, indicating the benefits of decoupled weight decay persist in the DP setting. Counterintuitively, the bias-corrected variant, DP-AdamW-BC, consistently led to worse performance than DP-AdamW across experiments, challenging the assumption that such bias correction is universally beneficial when combined with decoupled weight decay. Overall, the work positions DP-AdamW as a promising and effective optimizer for privacy-preserving deep learning, offering improved utility over existing methods on diverse classification tasks.

\section*{Acknowledgements}

The authors thank Salil Vadhan for his teachings on differential privacy, which inspired this paper. In addition, the authors thank Salil Vadhan and Zachary Ratliff for their thoughtful discussions and helpful suggestions for this manuscript.

\section*{Impact Statement}

This research investigates DP-AdamW and DP-AdamW-BC, more effective optimizers for differentially private deep learning, significantly advancing the protection of sensitive data. The key positive impact is enhanced data privacy, enabling AI in critical sectors like healthcare and finance with stronger safeguards. DP-AdamW's improved performance over existing private optimizers can lead to more useful and accurate AI models that respect user privacy, encouraging broader adoption of privacy-preserving practices and unlocking AI applications on sensitive datasets. However, challenges include the persistent privacy-utility trade-off and the risk of misinterpreting privacy guarantees. Ultimately, this work provides better tools for private AI, contributing to more trustworthy and responsible systems, though continued research and clear communication are vital for realizing its full positive potential.

\bibliography{references, bibliography}

@misc{wang_glue_2019,
	title = {{GLUE}: {A} {Multi}-{Task} {Benchmark} and {Analysis} {Platform} for {Natural} {Language} {Understanding}},
	shorttitle = {{GLUE}},
	url = {http://arxiv.org/abs/1804.07461},
	doi = {10.48550/arXiv.1804.07461},
	urldate = {2025-04-10},
	publisher = {arXiv},
	author = {Wang, Alex and Singh, Amanpreet and Michael, Julian and Hill, Felix and Levy, Omer and Bowman, Samuel R.},
	month = feb,
	year = {2019},
	note = {arXiv:1804.07461 [cs]},
}

@inproceedings{abadi2016deep,
  title={Deep learning with differential privacy},
  author={Abadi, Martin and Chu, Andy and Goodfellow, Ian and McMahan, H Brendan and Mironov, Ilya and Talwar, Kunal and Zhang, Li},
  booktitle={Proceedings of the 2016 ACM SIGSAC conference on computer and communications security},
  pages={308--318},
  year={2016}
}

@article{loshchilov2017decoupled,
  title={Decoupled weight decay regularization},
  author={Loshchilov, Ilya and Hutter, Frank},
  journal={arXiv preprint arXiv:1711.05101},
  year={2017}
}

@inproceedings{carlini2022membership,
  title={Membership inference attacks from first principles},
  author={Carlini, Nicholas and Chien, Steve and Nasr, Milad and Song, Shuang and Terzis, Andreas and Tramer, Florian},
  booktitle={2022 IEEE symposium on security and privacy (SP)},
  pages={1897--1914},
  year={2022},
  organization={IEEE}
}

@inproceedings{carlini2021extracting,
  title={Extracting training data from large language models},
  author={Carlini, Nicholas and Tramer, Florian and Wallace, Eric and Jagielski, Matthew and Herbert-Voss, Ariel and Lee, Katherine and Roberts, Adam and Brown, Tom and Song, Dawn and Erlingsson, Ulfar and others},
  booktitle={30th USENIX security symposium (USENIX Security 21)},
  pages={2633--2650},
  year={2021}
}

@inproceedings{balle2022reconstructing,
  title={Reconstructing training data with informed adversaries},
  author={Balle, Borja and Cherubin, Giovanni and Hayes, Jamie},
  booktitle={2022 IEEE Symposium on Security and Privacy (SP)},
  pages={1138--1156},
  year={2022},
  organization={IEEE}
}

@inproceedings{dwork2006calibrating,
  title={Calibrating noise to sensitivity in private data analysis},
  author={Dwork, Cynthia and McSherry, Frank and Nissim, Kobbi and Smith, Adam},
  booktitle={Theory of Cryptography: Third Theory of Cryptography Conference, TCC 2006, New York, NY, USA, March 4-7, 2006. Proceedings 3},
  pages={265--284},
  year={2006},
  organization={Springer}
}

@misc{défossez2022simpleconvergenceproofadam,
      title={A Simple Convergence Proof of Adam and Adagrad}, 
      author={Alexandre Défossez and Léon Bottou and Francis Bach and Nicolas Usunier},
      year={2022},
      eprint={2003.02395},
      archivePrefix={arXiv},
      primaryClass={stat.ML},
      url={https://arxiv.org/abs/2003.02395}, 
}

@article{krizhevsky2009learning,
  title={Learning multiple layers of features from tiny images},
  author={Krizhevsky, Alex and Hinton, Geoffrey and others},
  year={2009},
  publisher={Toronto, ON, Canada}
}

@inproceedings{li2022private,
  title={Private adaptive optimization with side information},
  author={Li, Tian and Zaheer, Manzil and Reddi, Sashank and Smith, Virginia},
  booktitle={International Conference on Machine Learning},
  pages={13086--13105},
  year={2022},
  organization={PMLR}
}

@article{zhou2020private,
  title={Private stochastic non-convex optimization: Adaptive algorithms and tighter generalization bounds},
  author={Zhou, Yingxue and Chen, Xiangyi and Hong, Mingyi and Wu, Zhiwei Steven and Banerjee, Arindam},
  journal={arXiv preprint arXiv:2006.13501},
  year={2020}
}

@misc{brown_language_2020,
	title = {Language {Models} are {Few}-{Shot} {Learners}},
	url = {http://arxiv.org/abs/2005.14165},
	doi = {10.48550/arXiv.2005.14165},
	urldate = {2025-04-30},
	publisher = {arXiv},
	author = {Brown, Tom B. and Mann, Benjamin and Ryder, Nick and Subbiah, Melanie and Kaplan, Jared and Dhariwal, Prafulla and Neelakantan, Arvind and Shyam, Pranav and Sastry, Girish and Askell, Amanda and Agarwal, Sandhini and Herbert-Voss, Ariel and Krueger, Gretchen and Henighan, Tom and Child, Rewon and Ramesh, Aditya and Ziegler, Daniel M. and Wu, Jeffrey and Winter, Clemens and Hesse, Christopher and Chen, Mark and Sigler, Eric and Litwin, Mateusz and Gray, Scott and Chess, Benjamin and Clark, Jack and Berner, Christopher and McCandlish, Sam and Radford, Alec and Sutskever, Ilya and Amodei, Dario},
	month = jul,
	year = {2020},
	note = {arXiv:2005.14165 [cs]},
}

@misc{ho_denoising_2020,
	title = {Denoising {Diffusion} {Probabilistic} {Models}},
	url = {http://arxiv.org/abs/2006.11239},
	doi = {10.48550/arXiv.2006.11239},
	urldate = {2025-04-30},
	publisher = {arXiv},
	author = {Ho, Jonathan and Jain, Ajay and Abbeel, Pieter},
	month = dec,
	year = {2020},
	note = {arXiv:2006.11239 [cs]},
}

@misc{li_large_2022,
	title = {Large {Language} {Models} {Can} {Be} {Strong} {Differentially} {Private} {Learners}},
	url = {http://arxiv.org/abs/2110.05679},
	doi = {10.48550/arXiv.2110.05679},
	urldate = {2025-04-30},
	publisher = {arXiv},
	author = {Li, Xuechen and Tramèr, Florian and Liang, Percy and Hashimoto, Tatsunori},
	month = nov,
	year = {2022},
	note = {arXiv:2110.05679 [cs]},
}

@misc{daigavane_node-level_2022,
	title = {Node-{Level} {Differentially} {Private} {Graph} {Neural} {Networks}},
	url = {http://arxiv.org/abs/2111.15521},
	doi = {10.48550/arXiv.2111.15521},
	urldate = {2025-04-10},
	publisher = {arXiv},
	author = {Daigavane, Ameya and Madan, Gagan and Sinha, Aditya and Thakurta, Abhradeep Guha and Aggarwal, Gaurav and Jain, Prateek},
	month = aug,
	year = {2022},
	note = {arXiv:2111.15521 [cs]},
}

@misc{hu_open_2021,
	title = {Open {Graph} {Benchmark}: {Datasets} for {Machine} {Learning} on {Graphs}},
	shorttitle = {Open {Graph} {Benchmark}},
	url = {http://arxiv.org/abs/2005.00687},
	doi = {10.48550/arXiv.2005.00687},
	urldate = {2025-04-10},
	publisher = {arXiv},
	author = {Hu, Weihua and Fey, Matthias and Zitnik, Marinka and Dong, Yuxiao and Ren, Hongyu and Liu, Bowen and Catasta, Michele and Leskovec, Jure},
	month = feb,
	year = {2021},
	note = {arXiv:2005.00687 [cs]},
}

@inproceedings{bowman_large_2015,
	address = {Lisbon, Portugal},
	title = {A large annotated corpus for learning natural language inference},
	url = {https://aclanthology.org/D15-1075/},
	doi = {10.18653/v1/D15-1075},
	urldate = {2025-04-10},
	booktitle = {Proceedings of the 2015 {Conference} on {Empirical} {Methods} in {Natural} {Language} {Processing}},
	publisher = {Association for Computational Linguistics},
	author = {Bowman, Samuel R. and Angeli, Gabor and Potts, Christopher and Manning, Christopher D.},
	editor = {Màrquez, Lluís and Callison-Burch, Chris and Su, Jian},
	month = sep,
	year = {2015},
	pages = {632--642},
}

@misc{tang_dp-adambc_2023,
	title = {{DP}-{AdamBC}: {Your} {DP}-{Adam} {Is} {Actually} {DP}-{SGD} ({Unless} {You} {Apply} {Bias} {Correction})},
	shorttitle = {{DP}-{AdamBC}},
	url = {http://arxiv.org/abs/2312.14334},
	doi = {10.48550/arXiv.2312.14334},
	urldate = {2025-04-10},
	publisher = {arXiv},
	author = {Tang, Qiaoyue and Shpilevskiy, Frederick and Lécuyer, Mathias},
	month = dec,
	year = {2023},
	note = {arXiv:2312.14334 [cs]},
}
\bibliographystyle{icml2025}

\newpage
\appendix
\onecolumn

\section{Privacy Guarantees: Proofs}  
\label{privacy_appendix_a}

We provide the proof of Theorem \ref{privacy_guarantee} below. 

\begin{proof}[Proof of Theorem \ref{privacy_guarantee}]

We exactly follow the proof of the Proposition 5 in \cite{tang_dp-adambc_2023}. Let $PrivitizeGradient(\theta, X, y, C, \sigma, B)$ denote the function whose output is the noised gradient $\tilde g = \frac1B\left( \sum_i \mathrm{clip}(g(x_i) + \mathcal N(0, \sigma^2 C^2)\right)$. Critically, this is the same function as used in DP-SGD and DP-AdamBC. Suppose that these algorithms are differentially private with privacy analysis given by $\phi(T, \mathbf \theta_i) = (\epsilon,\delta)$. Then note that the update of both DP-AdamW and DP-AdamW-BC do not involve any private information beyond the noised gradients. By the adaptive postprocessing property of DP, it follows that DP-AdamW and DP-AdamW-BC are also $(\epsilon, \delta)$-DP. Thus the privacy analysis of DP-AdamW and DP-AdamW-BC are also given by $\phi(T, \mathbf \theta_i)$, completing the proof.
\end{proof}

\begin{remark}
    Recall from Theorem 2 of \cite{tang_dp-adambc_2023} that DP-SGD has privacy guarantee $(\epsilon,\delta)$ whenever the noise satisfies $\sigma \geq \frac{c_2 q\sqrt{T \log \frac 1\delta}}{\epsilon}$. This thus provides a concrete privacy guarantee for DP-AdamW and DP-AdamW-BC as well. 
\end{remark}

\newpage

\section{Convergence Guarantees: Proofs}
\label{convergence_appendix_b}
We now provide proofs for the convergence results presented in the main body of the paper. 

\subsection{Proof of Theorem \ref{conv1}}
We prove the convergence result given in Theorem \ref{conv1} below.
\begin{proof}[Proof of Theorem \ref{conv1}]
    We will only prove the result for DP-AdamW; the exact same proof technique will suffice for DP-AdamW-BC. As in \cite{tang_dp-adambc_2023}, define the notation $u_t = \frac{\nabla_i f_t^p(\theta_{t-1})}{\sqrt{v_t^p + \epsilon_0}}$ denote the main `update term` in DP-Adam; that is, $\nabla_i f_t^p(\theta_{t-1}) = \tilde m_t$ and $v_t^p = \tilde v_t$. We first prove a bound on the expected magnitude of the parameters, then proceed as in \cite{tang_dp-adambc_2023}.

    We first note that Lemma 3 and Corollary 1 in Appendix F of \cite{tang_dp-adambc_2023} hold, with the proof remaining identical. Moreover, identical arguments as in the proof of Proposition 6 of \cite{tang_dp-adambc_2023} yield the inequalities\footnote{These results are directly shown and/or implicit in the referenced proof; the proof is immediately adaptable.} \begin{equation}\label{ubound}\sum_{t = 0}^{T-1} \mathbb E[\|u_t\|^2] \leq d\left(\ln \left(1 + \frac{C^2 + \Phi}{(1-\beta_2)\epsilon_0}\right) - T\ln \beta_2\right)\end{equation} and moreover \begin{align}
        \label{mainbound1} \sum_{t = 1}^T \eta_t \mathbb E\left[\nabla F(\theta_{t-1})^\top \frac{\nabla f_t^p(\theta_{t-1})}{\sqrt{v_t^p + \epsilon_0}}\right] &\geq -\frac{\eta}{2 (\delta_0+C_1)} \sum_{t=1}^{T}\mathbb E [\|\nabla F(\theta_{t-1})\|^2] \\ &+ \frac{2\eta\sqrt{C^2 + \Phi}}{\sqrt{1 - \beta_2}} \sum_{t=0}^{T-1} \mathbb E[\|u_t\|]^2.
    \end{align} In particular, the first statement follows from Lemma 5.2 in \cite{défossez2022simpleconvergenceproofadam}, and the second by using Lemmas 5.1 and 5.2 of the aforementioned work. 


    We now proceed as follows. First, we bound the magnitude of the parameters $\theta_t$ in expectation. Let $$R = d \left(\ln \left(1 + \frac{C^2 + \Phi}{(1 - \beta_2) \epsilon_0}\right) - T\ln \beta_2\right).$$ 

    \textbf{Claim:} $\mathbb E\|\theta_t\|^2 \leq \|\theta_0\|^2 + c(\lambda) R$, where $c(\lambda)$ is a constant depending on $\lambda$ and all other constant parameters. \footnote{Here $c(\lambda)$ notation encapsulates all constant parameters, as well as $\theta_0$.}

    \textbf{Proof:} The main idea is to combine the recursive definition of $\theta_t$ with the AM-GM inequality and \ref{mainbound1}, Let $x_t = \mathbb E \|\theta_t\|^2$ and write \begin{align*}
        x_t = \mathbb E\|\theta_t\|^2 &= (1 - \eta_t \lambda)^2 \mathbb E\|\theta_{t-1}\|^2 + \eta_t^2 \mathbb E\|u_t\|^2 + 2\eta_t (1 - \eta_t\lambda) \mathbb E[\langle \theta_{t-1}, u_t\rangle] \\ &\leq (1-\eta_t \lambda)^2 x_{t-1} + \eta_t^2 \mathbb E\|u_t\|^2 + 2\eta_t (1 - \eta_t \lambda) \left(\frac \lambda 2 \mathbb E\|\theta_{t-1}\|^2 + \frac{1}{2\lambda} \mathbb E\|u_t\|^2\right) \\ &= \left[\eta_t^2 + \frac{\eta_t (1 - \eta_t\lambda)}{\lambda}\right] \mathbb E\|u_t\|^2 + (1-\eta_t\lambda) \mathbb E\|\theta_{t-1}\|^2 \\ &\leq \left(\eta_t^2 + \frac{\eta_t}{\lambda}\right) \mathbb E\|u_t\|^2 + x_{t-1},
    \end{align*} where we have used the fact that $\langle x,y\rangle = \sum x_i y_i \leq \lambda \sum x_i^2 + \frac 1\lambda \sum y_i^2$. It follows by summing the above inequality for all $t$, telescoping, and applying inequality \ref{ubound} that \begin{align*}x_t &\leq \|\theta_0\|^2 + (1 + \frac 1\lambda)\left(\max_t (\eta_t + \eta_t^2)d \left(\ln \left(1 + \frac{C^2 + \Phi}{(1 - \beta_2)\epsilon_0}\right) - T \ln \beta_2 \right)\right) \\ &=\|\theta_0\|^2 + c(\lambda) \max_t (\eta_t+\eta_t^2) R.\end{align*} This completes the proof of the claim. 

    Now we return to the proof of the main theorem. We follow the rough outline used by \cite{tang_dp-adambc_2023}. The main idea is to use the Lipschitzness of the gradient to bound the difference $F(\theta_t) - F(\theta_{t-1})$. One can then telescope this difference and combine the claim with inequalities \ref{ubound} and \ref{mainbound1} and some algebraic manipulations.
    
    Let us write $\theta_t = \theta_{t-1} - \eta_t (u_t + \lambda_t \theta_{t-1})$. Using the Lipschitzness assumption of the gradient, we have \begin{align}\label{111}F(\theta_t) &\leq F(\theta_{t-1}) - \eta_t \langle \nabla F(\theta_{t-1}), u_t\rangle - \eta_t \lambda \langle \nabla F(\theta_{t-1}), \theta_{t-1} \rangle + \frac{\eta_t^2L}{2} \|u_t\|^2 \\ &+ \frac{\eta_t^2L}{2} \lambda^2 \|\theta_{t-1}\|^2 + \eta_t^2 L\lambda \langle u_t, \theta_{t-1}\rangle.\end{align} Taking expectations, we find by the above claim that $\mathbb E\|\theta_{t-1}\|^2 = \|\theta_0\|^2 + c(\lambda) \max_t (\eta_t+\eta_t^2) R$ and hence $$\mathbb E|\langle \nabla F(\theta_{t-1}), \theta_{t-1}\rangle| \leq \frac12 (\mathbb E \|\nabla F(\theta_{t-1})\|^2 + \mathbb E \|\theta_{t-1}\|^2) \leq \frac12(C_1^2+\|\theta_0\|^2 + c(\lambda)\max_t (\eta_t+\eta_t^2) R).$$ Moreover, recall that $$\mathbb E |\langle u_t, \theta_{t-1}\rangle| \leq \frac12 (\mathbb E \|u_t\|^2 + \mathbb E \|\theta_{t-1}\|^2) \leq \frac12 \mathbb E \|u_t\|^2 + \frac12 (\|\theta_0\|^2 + c(\lambda) \max_t (\eta_t+\eta_t^2) R).$$
    
    Therefore, taking expectations of \ref{111}, summing over all $t$, combining inequality \ref{mainbound1} with the above, and lastly using Corollary 1 of \cite{tang_dp-adambc_2023}, we find that \begin{align*}
        \mathbb E F(\theta_T) &\leq F(\theta_0) - \frac{\eta}{2(\delta_0+C_1)} \sum_{t = 1}^T \mathbb E\|\nabla F(\theta_{t-1})\|^2 \\ &+ \left[\frac{2\eta \sqrt{C^2 + \Phi}}{\sqrt{1 - \beta_2}} + \frac{\eta^2 L(1 + \lambda)}{2 (1-\beta_2)}\right] \sum_{0}^{T-1} \mathbb E\|u_t\|^2 \\ &+ \frac12 (C_1^2 + \|\theta_0\|^2 + c(\lambda) \max_t (\eta_t+\eta_t^2) R) \left(\lambda \sum_{t=1}^T \eta_t  +L(\lambda+\lambda^2) \sum_{t=1}^{T} \eta_t^2 \right).
    \end{align*} Rearranging, noting that $\mathbb E F(\theta_T) \geq F_*$, and applying inequality \ref{ubound}, we find that \begin{align*}
        \frac1T \sum_{0}^{T-1} \mathbb E\|\nabla F(\theta_{t-1})\|^2 &\leq \frac{2(\delta_0+C_1)}{\eta} \left(\frac{F(\theta_0) - F_*}{T}\right) \\ &+ \left(\frac{4d(C^2 + \Phi)}{\sqrt {1 - \beta_2}} + \frac{\eta dL \sqrt{C^2 + \Phi}(1 + \lambda)}{1 - \beta_2}\right) \cdot \frac RT\\ &+ \frac1{2T} \left(C_1^2 + \|\theta_0\|^2 + c(\lambda) \max_t (\eta_t+\eta_t^2) R\right) \left(\lambda \sum_{t=1}^T \eta_t + \frac L2 (\lambda + \lambda^2) \sum_{t=1}^T \eta_t^2\right). 
    \end{align*} This is exactly the desired inequality and completes the proof of the theorem. 
\end{proof}
\begin{remark}
    We owe a quick explanation of the starting line of the proof; for DP-AdamW-BC, one only needs to replace the inequalities given by \ref{mainbound1} and \ref{ubound} with the corresponding bounds in the proof of Proposition 6 of \cite{tang_dp-adambc_2023}; the remainder of the proof is the exact same. 

    For an interpretation of the asymptotic growth of this result, see the remark following the initial theorem statement. 
\end{remark}

\subsection{Proof of Theorem \ref{conv2}}
We prove the convergence result given in Theorem \ref{conv2} below.
\begin{proof}[Proof of Theorem \ref{conv2}]
We will only prove the result for DP-AdamW; the exact same proof technique will suffice for DP-AdamW-BC. As in \cite{tang_dp-adambc_2023}, define the notation $u_t = \frac{\nabla_i f_t^p(\theta_{t-1})}{\sqrt{v_t^p + \epsilon_0}}$ denote the main `update term` in DP-Adam; that is, $\nabla_i f_t^p(\theta_{t-1}) = \tilde m_t$ and $v_t^p = \tilde v_t$. We first prove a bound on the expected magnitude of the parameters, then proceed as in \cite{tang_dp-adambc_2023}.

    We first note that Lemma 3 and Corollary 1 in Appendix F of \cite{tang_dp-adambc_2023} hold, with the proof remaining identical. Moreover, identical arguments as in the proof of Proposition 7 of \cite{tang_dp-adambc_2023} yield the inequalities\footnote{These results are directly shown and/or implicit in the referenced proof; the proof is immediately adaptable.} \begin{equation}\label{uboundMOMENTUM}\sum_{t = 0}^{T-1} \mathbb E[\|u_t\|^2] \leq d\left(\ln \left(1+ \frac{\delta_0^2}{\epsilon_0 (1 - \beta_2)}\right) - T \log \beta_2\right) \end{equation} and moreover
        
\begin{align}\label{mainbound1MOMENTUM}
\mathbb{E}\left[\nabla F(\theta_{t-1})^{\top} u_t\right]
&\ge \frac{1}{2} \sum_{i\in[d]} \sum_{k=0}^{t-1} \beta_1^{k}
       \mathbb{E}\left[
         \frac{\nabla_i F(\theta_{t-k-1})}{\sqrt{\tilde v_{t,k+1,i}-\Phi}}
       \right] \\
&- \frac{\sqrt{1-\beta_1}\eta_t^{2}L^{2}}{4\delta_0}
       \sum_{l=1}^{t-1} \|u_{t-l}\|^{2}
       \sum_{k=1}^{l-1} \beta_1^{k}\sqrt{k} \\
&- \frac{3\delta_0}{\sqrt{1-\beta_1}}
       \sum_{k=0}^{t-1} \left(\frac{\beta_1}{\beta_2}\right)^{2}
       \|u_{t-k}\|^{2}.
\end{align}

In particular, the first statement follows from Lemma 5.2 in \cite{défossez2022simpleconvergenceproofadam}, and the second by using Lemmas 5.1 and 5.2 of the aforementioned work. 


    We now proceed as follows. First, we bound the magnitude of the parameters $\theta_t$ in expectation. Let $$R = d\left(\ln \left(1+ \frac{\delta_0^2}{\epsilon_0 (1 - \beta_2)}\right) - T \log \beta_2\right) .$$ 

    \textbf{Claim:} $\mathbb E\|\theta_t\|^2 \leq \|\theta_0\|^2 + c(\lambda) R$, where $c(\lambda)$ is a constant depending on $\lambda$ and all other constant parameters. \footnote{Here $c(\lambda)$ notation encapsulates all constant parameters, as well as $\theta_0$.}

    \textbf{Proof:} The main idea is to combine the recursive definition of $\theta_t$ with the AM-GM inequality and \ref{mainbound1}, Let $x_t = \mathbb E \|\theta_t\|^2$ and write \begin{align*}
        x_t = \mathbb E\|\theta_t\|^2 &= (1 - \eta_t \lambda)^2 \mathbb E\|\theta_{t-1}\|^2 + \eta_t^2 \mathbb E\|u_t\|^2 + 2\eta_t (1 - \eta_t\lambda) \mathbb E[\langle \theta_{t-1}, u_t\rangle] \\ &\leq (1-\eta_t \lambda)^2 x_{t-1} + \eta_t^2 \mathbb E\|u_t\|^2 + 2\eta_t (1 - \eta_t \lambda) \left(\frac \lambda 2 \mathbb E\|\theta_{t-1}\|^2 + \frac{1}{2\lambda} \mathbb E\|u_t\|^2\right) \\ &= \left[\eta_t^2 + \frac{\eta_t (1 - \eta_t\lambda)}{\lambda}\right] \mathbb E\|u_t\|^2 + (1-\eta_t\lambda) \mathbb E\|\theta_{t-1}\|^2 \\ &\leq \left(\eta_t^2 + \frac{\eta_t}{\lambda}\right) \mathbb E\|u_t\|^2 + x_{t-1},
    \end{align*}where we have used the fact that $\langle x,y\rangle = \sum x_i y_i \leq \lambda \sum x_i^2 + \frac 1\lambda \sum y_i^2$. It follows by summing across all $t$, telescoping, and applying inequality \ref{ubound} that \begin{align*}x_t &\leq \|\theta_0\|^2 + (1 + \frac 1\lambda)\left(\max_t (\eta_t + \eta_t^2)d\left(\ln \left(1+ \frac{\delta_0^2}{\epsilon_0 (1 - \beta_2)}\right) - T \log \beta_2\right) \right) \\ &=\|\theta_0\|^2 + c(\lambda) \max_t (\eta_t+\eta_t^2) R.\end{align*} This completes the proof of the claim. 

Now we return to the proof of the main theorem. We follow the rough outline used by \cite{tang_dp-adambc_2023}. The main idea is to use the Lipschitzness of the gradient to bound the difference $F(\theta_t) - F(\theta_{t-1})$. One can then telescope this difference and combine the claim with inequalities \ref{ubound} and \ref{mainbound1} and some algebraic manipulations.
    
    Let us write $\theta_t = \theta_{t-1} - \eta_t (u_t + \lambda_t \theta_{t-1})$. Using the Lipschitzness assumption of the gradient, we have \begin{align}\label{333} F(\theta_t) &\leq F(\theta_{t-1}) - \eta_t \langle \nabla F(\theta_{t-1}), u_t\rangle - \eta_t \lambda \langle \nabla F(\theta_{t-1}), \theta_{t-1} \rangle + \frac{\eta_t^2L}{2} \|u_t\|^2 \\ &+ \frac{\eta_t^2L}{2} \lambda^2 \|\theta_{t-1}\|^2 + \eta_t^2 L\lambda \langle u_t, \theta_{t-1}\rangle.\end{align} Taking expectations, we find by the above claim that $\mathbb E\|\theta_{t-1}\|^2 = \|\theta_0\|^2 + c(\lambda) \max_t (\eta_t+\eta_t^2) R$ and hence $$\mathbb E|\langle \nabla F(\theta_{t-1}), \theta_{t-1}\rangle| \leq \frac12 (\mathbb E \|\nabla F(\theta_{t-1})\|^2 + \mathbb E \|\theta_{t-1}\|^2) \leq \frac12(C_1^2+\|\theta_0\|^2 + c(\lambda)\max_t (\eta_t+\eta_t^2) R).$$ Moreover, recall that $$\mathbb E |\langle u_t, \theta_{t-1}\rangle| \leq \frac12 (\mathbb E \|u_t\|^2 + \mathbb E \|\theta_{t-1}\|^2) \leq \frac12 \mathbb E \|u_t\|^2 + \frac12 (\|\theta_0\|^2 + c(\lambda) \max_t (\eta_t+\eta_t^2) R).$$
    
Therefore, taking expectations of \ref{333}, summing over all $t$, combining inequality \ref{mainbound1MOMENTUM} with the above, and lastly using Corollary 1 of \cite{tang_dp-adambc_2023} and the simple inequality $\eta_t \leq \eta_T$, we find that \begin{align*}\frac{\sum_{t=1}^{T}\frac{\eta_t}{\Omega_t}\sum_{k=0}^{t-1}\beta_1^{k}
      \mathbb{E}\left[\left\|\nabla F(\theta_{t-k-1})\right\|^{2}\right]}{2(\delta_0 + C_1)}
&\le F(\theta_{0})-F_{*} \\
&+\frac{\eta_T^{2}L}{2}\sum_{t=1}^{T}\mathbb{E}\left[\left\|u_t\right\|^{2}\right] +\frac{\eta_T^{3}L^{2}\sqrt{1-\beta_1}}{4(\delta+C_1)}
        \sum_{t=1}^{T}\sum_{l=1}^{t-1}\beta_1^{l}\sqrt{l} \\
&+\frac{3\eta_T(\delta+C_1)}{\sqrt{1-\beta_1}}
        \sum_{t=1}^{T}\sum_{k=0}^{t-1}
        \left(\frac{\beta_1}{\beta_2}\right)^{k}
        \mathbb{E}\left[\left\|u_{t-k}\right\|^{2}\right].
\end{align*}
where $\Omega_t = \sqrt{\sum_{j=0}^{t-1} \beta_2^j}$. Rearranging, noting that $\mathbb E F(\theta_T) \geq F_*$, and applying inequality \ref{uboundMOMENTUM}, we find that \begin{align*}
        \mathbb E[\|\nabla F(\theta_\tau)\|^2] &\leq \frac{2(\delta_0 + C_1)(F(\theta_0) - F_*)}{\eta \tilde T} +ER,
    \end{align*} 
    
where the left hand side expectation is taken with respect to the sample $\mathbb P(\tau = t) \propto 1 - \beta_1^{T-t}$, $\tilde T = T - \frac{\beta_1}{1 - \beta_1}$, and $$E = \frac{\eta dL (1 - \beta_1)(\delta_0+C_1)}{(1 - \beta_1 / \beta_2)(1 - \beta_2)} + \frac{2\eta^2 dL^2 \beta_1}{(1 - \beta_1/\beta_2)(1 - \beta_2)^{3/2}} + \frac{12 d(\delta_0 + C_1)^2\sqrt{1-\beta_1}}{(1 - \beta_1/\beta_2)^{3/2} \sqrt{1 - \beta_2}},$$ as in the statement of the theorem. This is exactly the desired inequality and completes the proof of the theorem. 
\end{proof}
\begin{remark}
Similar to the proof of Theorem \ref{conv1}, we note that for DP-AdamW-BC, one only needs to replace the inequalities given by \ref{mainbound1MOMENTUM} and \ref{uboundMOMENTUM} with the corresponding bounds in the proof of Proposition 7 of \cite{tang_dp-adambc_2023}; the remainder of the proof is the exact same. 
\end{remark}

\newpage
\section{Experimental Details}
\label{app:experimental_details}

\subsection{Image Classification}

We train a 5-layer CNN model to perform image classification on the CIFAR-10 dataset \citep{krizhevsky2009learning}. The CIFAR-10 dataset includes 60,000 32x32 images, with 6,000 images for each of 10 classes: airplane, automobile, bird, cat, deer, dog, frog, horse, ship, truck. We first train on the training set containing 50,000 images, then we test on the test set of the 10,000 remaining images. We train the model from scratch with randomized initial model parameters.

To ensure a fair comparison and optimal performance for each optimizer, we conduct extensive hyperparameter sweeps. Following preliminary exploration and to maintain consistency with \citet{tang_dp-adambc_2023}, we fix the $L_2$-norm clipping bound at $C=1.0$ and utilized Adam stability parameters of $5 \times 10^{-8}$. Our primary hyperparameter sweep focuses on optimizing learning rate ($\alpha$) and weight decay ($\lambda$) for each target privacy level ($\epsilon \in \{1, 3, 7\}$) and for both DP-AdamW variants (with and without bias correction). We explore a range of learning rates and weight decay values suitable for each privacy setting based on values commonly used in related DP deep learning literature. All experiments are run for 70 epochs using a batch size of 1024, and results are averaged over 5 trials with different random seeds. We select weight decay $\lambda = 1 \times 10^{-5}$ after sampling over $\lambda \in \{0.01, 0.001, 1 \times 10^{-4}, 1 \times 10^{-5}\}$ and optimizing training loss.

Figures \ref{fig: cifar10_1_loss} and \ref{fig: cifar10_1_accuracy} in Appendix \ref{appendix:cifar} show the training loss and test accuracy over learning rates for $\epsilon = 1$. We select learning rates by examining the values that minimized training loss. We outline learning rates over $\epsilon$ values in Table \ref{table: cifar10_lr}.

\subsection{Text Classification}

We study natural-language understanding under differential privacy on the standard question-natural language inference (QNLI) task from the GLUE natural language benchmark~\citep{wang_glue_2019}. QNLI contains 115{,}669 (question,~sentence) pairs split into 105K/5.5K/5.5K (train/validation/test) examples and asks whether the sentence answers the corresponding question. Performance is reported as classification accuracy on the held-out test set. We use the DP-Adam and DP-Adam-BC results from \citet{tang_dp-adambc_2023} as a baseline comparison, although we independently fine-tune hyperparameters for our implementations of DP-AdamW and DP-AdamW-BC.

Following common practice for private fine-tuning (e.g.~\citealp{li_large_2022,tang_dp-adambc_2023}) we start from the \texttt{bert-base-cased} checkpoint with 110M parameters and freeze all layers except the last Transformer block, the pooler, and the task-specific classifier head, which yields approximately 7M trainable parameters. The model is trained for $10$ epochs with a batch size of $B=32$. Gaussian noise calculated to match the target privacy budget $\epsilon\!\in\!\{1,3,7\}$ is added through \textsc{Opacus}' \texttt{PrivacyEngine} and per-sample gradients are clipped using an $\ell_{2}$-norm threshold before aggregation and noise addition. To ensure compatibility, the same noise parameters are used across all optimizers.

For every combination of $(\epsilon,$ optimizer$)$ under consideration and following the approach of~\citep{tang_dp-adambc_2023}, we perform a grid search over hyperparameters: learning rate $\alpha \in [1\!\times\!10^{-6},5\!\times\!10^{-3}]$, weight decay $\lambda\in\{0,10^{-5},10^{-4}\}$, and gradient clipping bound $C\in\{0.05,0.1,0.2\}$. Additionally, after preliminary exploration, we fix $\gamma=10^{-8}$ for the minimum variance term in DP-AdamW-BC. For each privacy setting, the learning rates and weight decay values were chosen based on those commonly used in the DP deep learning literature. Moreover, the candidate values for the gradient clipping bound were those examined by \citep{tang_dp-adambc_2023} therefore we consider the same values for experimental consistency purposes. We optimize for standard cross-entropy loss during training. Note that this choice of loss function is correct because the specific problem is a classification task, where cross-entropy is the canonical choice, and it is differentiable meaning that it works nicely with the per-sample gradient computation that \textsc{Opacus} requires for differential privacy.

\subsection{Graph Node Classification}

We evaluated different optimizers on the graph node classification task using the ogbn-arxiv task from the Open Graph Benchmark (\citep{hu_open_2021}) and trained a 2-layer differentially private graph convolutional network model (\citep{daigavane_node-level_2022}) without per-layer clipping from scratch. We tested for \(\epsilon\in\{3,6,12\}\) across the optimizers. The total number of parameters is 173,695 and the specific architecture details can be found in the Appendix at Section \ref{appendix:gnn_section}.

For each new optimizer-$\epsilon$ pair, we did a hyperparameter sweep of 10 runs over a validation set using a Bayes-informed sampling over a log-uniform distribution to maximize validation accuracy. The results of this preliminary sweep are reported in Table \ref{tab:hyperparams_gnn_10} in the appendix. With this preliminary finding, we then adjusted the range to be swept to encompass our best guess of the best hyperparameters to run another sweep of 30 runs. The final range is \([10^{-1},10^{-6}]\) for the learning rate and \([10^{-1},10^{-7}]\) for the weight decay. The best learning rate and weight decay chosen from the 30-run sweep are reported in Table \ref{tab:hyperparams_gnn} in the Appendix. We noticed that as \(\epsilon\) increases, the best learning rate decreases while the weight decay generally, though not always, increases. We found that validation accuracy increases as \(\epsilon\) increases, which is expected since models perform better with looser privacy budgets. We attach the plots of the losses in Appendix \ref{appendix:gnn_section} for the case of DP-AdamW under \(\epsilon=3\).

\newpage

\section{Image Classification Experimental Details}
\label{appendix:cifar}

\begin{table}[h]
\centering
\caption{Learning rates for $\epsilon$ values on CIFAR-10}
\begin{tabular}{cccc}
\hline\hline
 & $\epsilon = 1$ & $\epsilon = 3$ & $\epsilon = 7$ \\
\hline
DP-AdamW & $0.0015$ & $0.0015$ & $0.002$ \\
DP-AdamW-BC & $1 \times 10^{-4}$ & $2 \times 10^{-4}$ & $4 \times 10^{-4}$ \\
\hline\hline
\end{tabular}
\label{table: cifar10_lr}
\end{table}

\begin{figure}[h]
\begin{center}
  \includegraphics[width=0.49\linewidth]{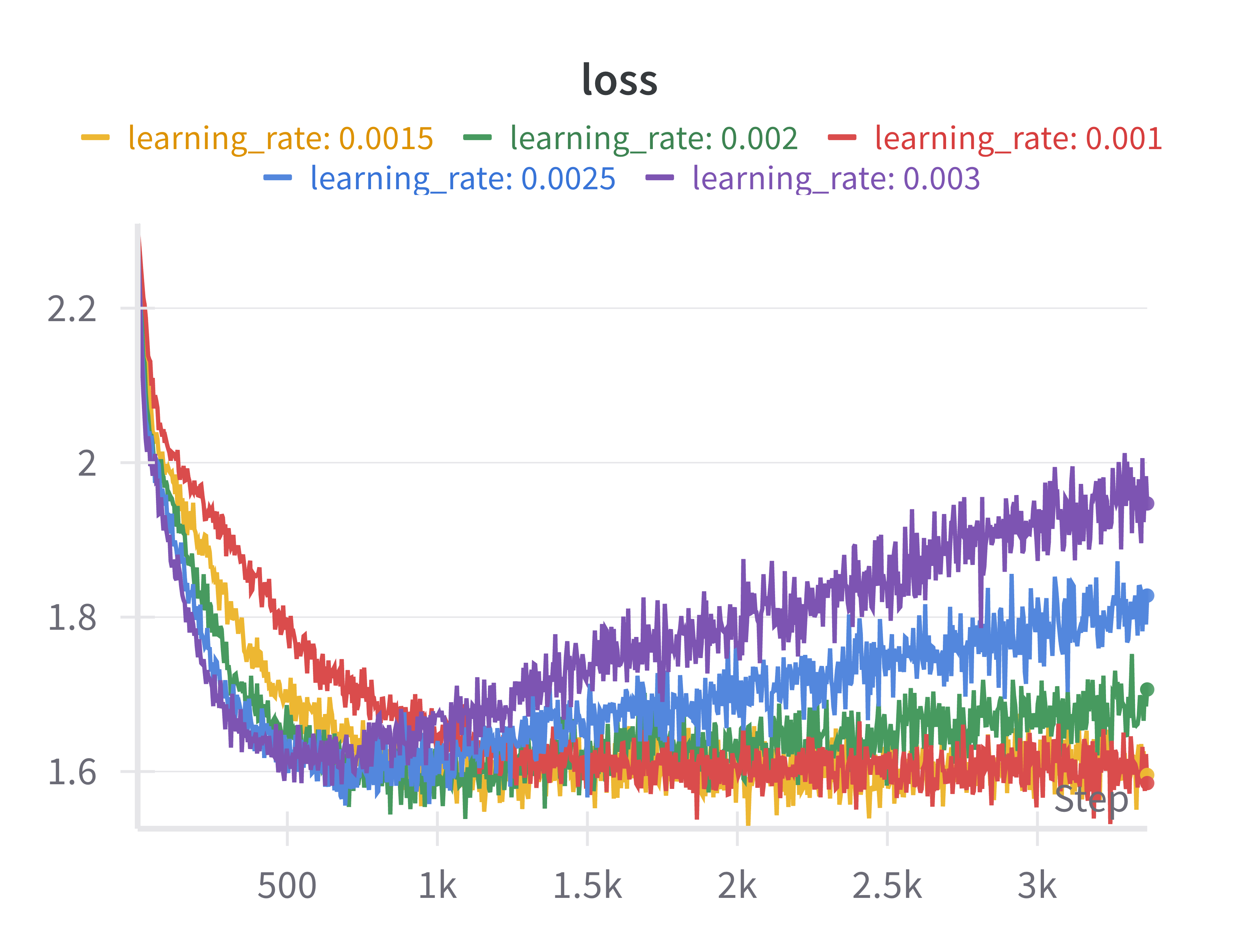}
  \includegraphics[width=0.49\linewidth]{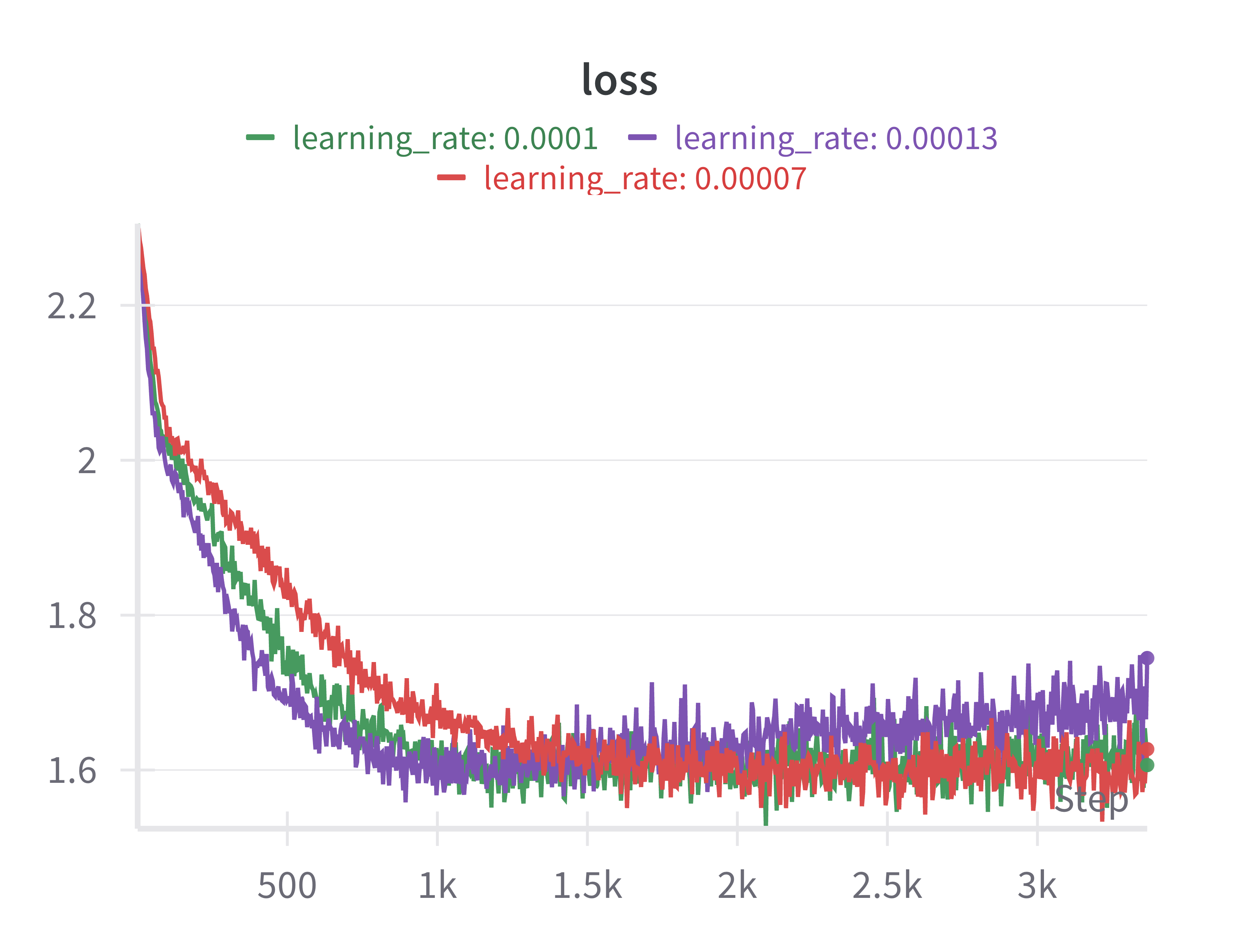}
\end{center}
\caption{Training CIFAR-10 for $\epsilon = 1$ across learning rates for DP-AdamW (left) and DP-AdamW-BC (right), with step on x-axis and training loss on y-axis}
\label{fig: cifar10_1_loss}
\end{figure}

\begin{figure}[h]
\begin{center}
  \includegraphics[width=0.49\linewidth]{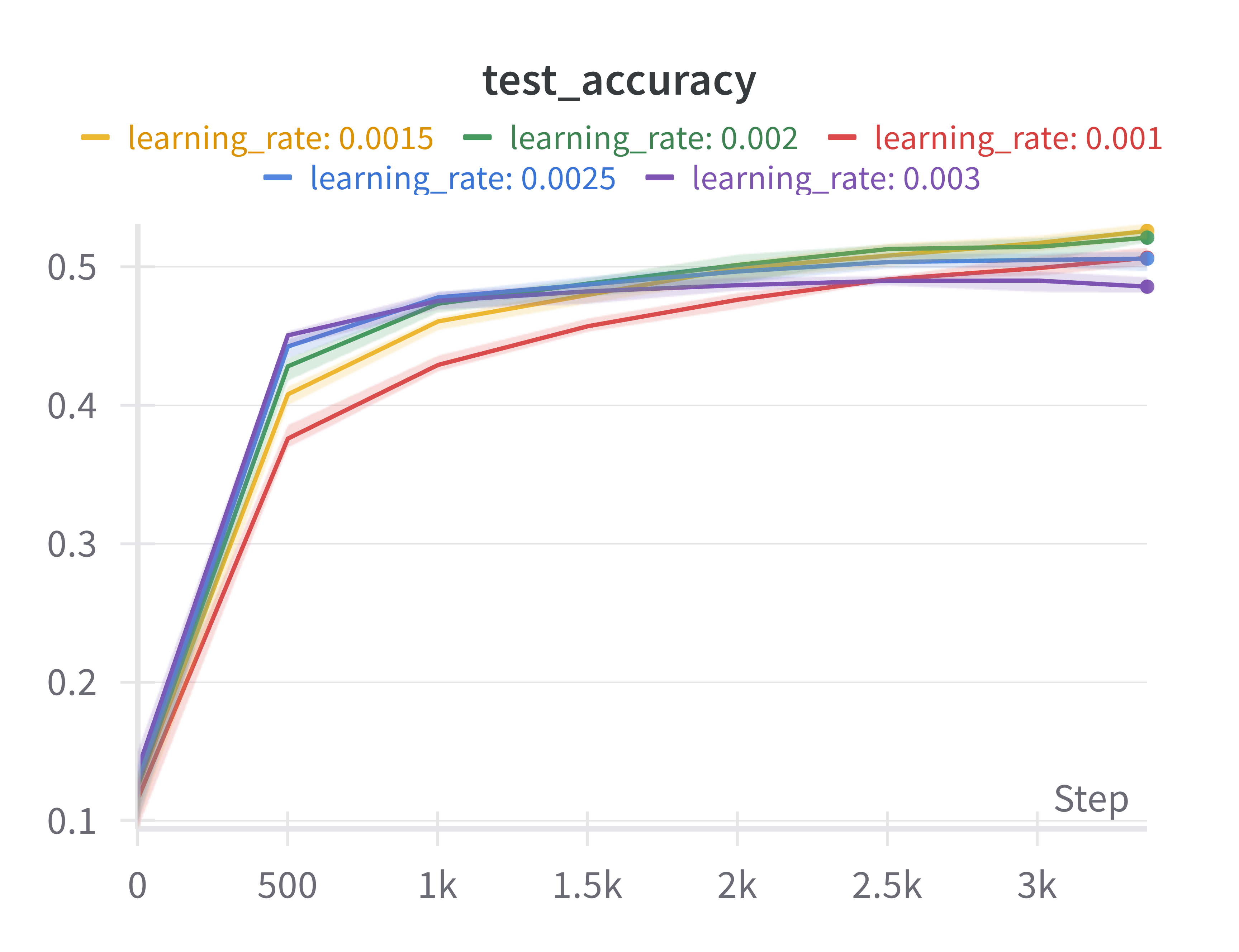}
  \includegraphics[width=0.49\linewidth]{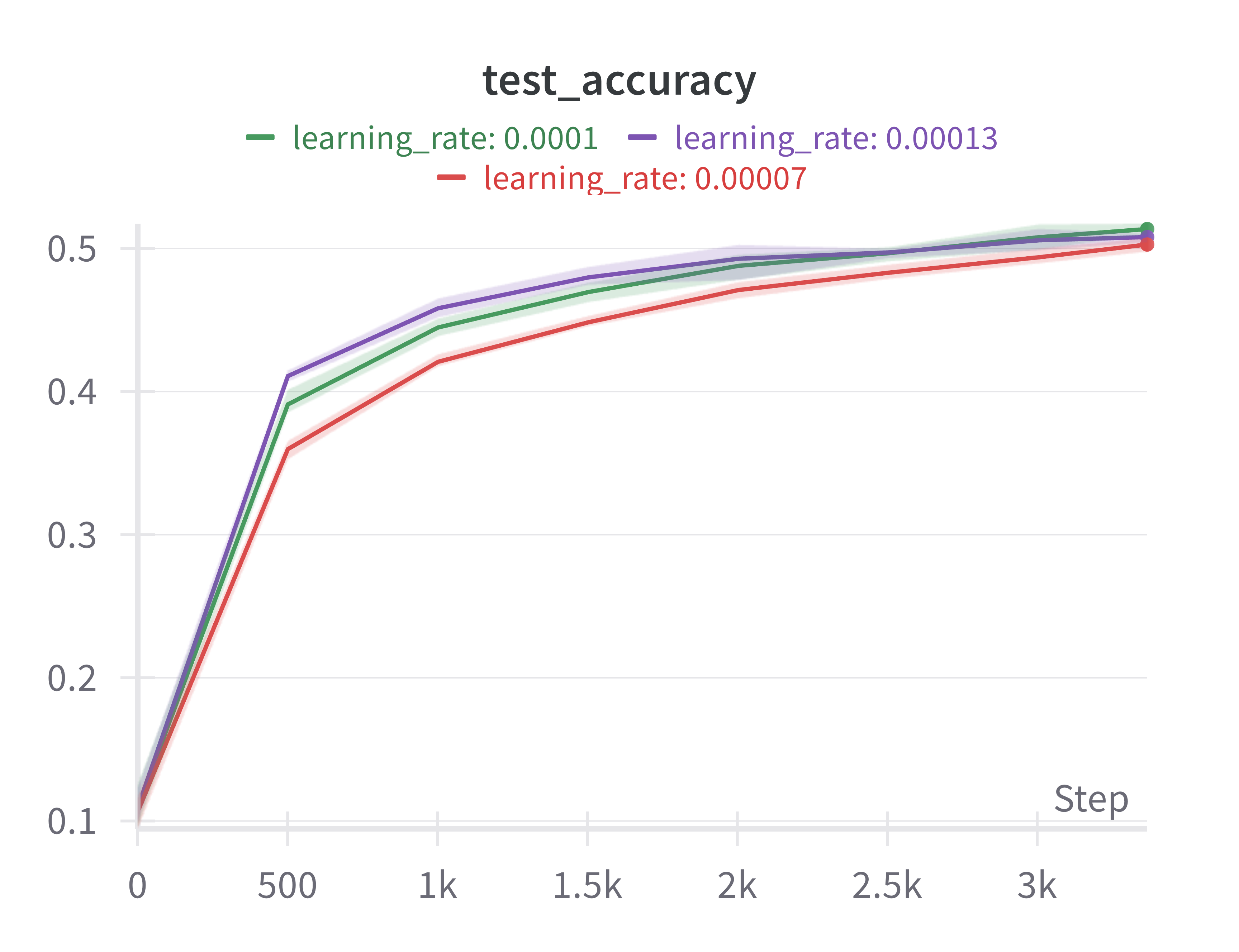}
\end{center}
\caption{Evaluating on CIFAR-10 for $\epsilon = 1$ across learning rates for DP-AdamW (left) and DP-AdamW-BC (right), with step on x-axis and test accuracy (proportion) on y-axis}
\label{fig: cifar10_1_accuracy}
\end{figure}

\newpage

\section{Graph Node Classification Experimental Details}
\label{appendix:gnn_section}

Below we provide more details on the training process of the graph node classification task on which the optimizers were evaluated. Table \ref{tab:gnn-params} shows the architecture of the DP-GCN used in Section \ref{gnn}, with a total of 173,695 parameters (694,780 bytes).

\begin{table}[h]
\centering
\caption{Model parameter names, shapes, and data types}
\label{tab:gnn-params}
\begin{tabular}{lll}
\toprule
\textbf{Parameter} & \textbf{Shape} & \textbf{Dtype} \\
\midrule
Core Layer 0: Dense Bias         & (255,)     & float32 \\
Core Layer 0: Dense Kernel       & (255, 255) & float32 \\
Decoder Layer 0: Dense Bias      & (255,)     & float32 \\
Decoder Layer 0: Dense Kernel    & (255, 255) & float32 \\
Decoder Layer 1: Dense Bias      & (40,)      & float32 \\
Decoder Layer 1: Dense Kernel    & (255, 40)  & float32 \\
Encoder Layer 0: Dense Bias      & (255,)     & float32 \\
Encoder Layer 0: Dense Kernel    & (128, 255) & float32 \\
\bottomrule
\end{tabular}
\end{table}

Table \ref{tab:hyperparams_gnn_10} shows the best hyperparameters for the GNN task identified after the initial 10 runs of the hyperparameter sweep.

\begin{table}[!ht]
\centering
\caption{Hyperparameters and validation performance for DP-AdamW and DP-AdamW-BC on \texttt{obgn-arxiv} node classification (best validation score of 10)}
\label{tab:hyperparams_gnn_10}
\begin{tabular}{c|c|c|c|c}
\hline\hline
\textbf{Optimizer} & \textbf{Epsilon} & \textbf{Learning Rate} & \textbf{Weight Decay} & \textbf{Val.\ Acc.\,(\%)} \\
\hline
\multirow{3}{*}{AdamW}
  & 3   & 0.01494606 & 0.00047264   & 51.32 \\
  & 6   & 0.00752970 & 0.00095908   & 54.44 \\
  & 12  & 0.00434687 & 0.00051222   & 56.16 \\
\hline
\multirow{3}{*}{AdamW-BC}
  & 3   & 0.00018378 & 0.0000026934 & 50.30 \\
  & 6   & 0.00009699 & 0.0000024191 & 54.28 \\
  & 12  & 0.00004261 & 0.00013755   & 55.86 \\
\hline\hline
\end{tabular}
\end{table}

Table \ref{tab:hyperparams_gnn} shows the best hyperparameters for the GNN task identified after the follow-up 30 runs of the hyperparameter sweep.

\begin{table}[!ht]
\centering
\caption{Hyperparameters and validation performance for DP-AdamW and DP-AdamW-BC on \texttt{obgn-arxiv} node classification (best validation score of 30)}
\label{tab:hyperparams_gnn}
\begin{tabular}{c|c|c|c|c}
\hline\hline
\textbf{Optimizer} & \textbf{Epsilon} & \textbf{Learning Rate} & \textbf{Weight Decay} & \textbf{Val.\ Acc.\,(\%)} \\
\hline
\multirow{3}{*}{DP-AdamW}
  & $\epsilon\approx3$  & 0.01246956 & 0.0000004780 & 52.41 \\
  & $\epsilon\approx6$  & 0.00806904 & 0.0000091133 & 55.18 \\
  & $\epsilon\approx12$ & 0.00427425 & 0.0000709134 & 56.27 \\
\hline
\multirow{3}{*}{DP-AdamW-BC}
  & $\epsilon\approx3$  & 0.00021068 & 0.0000324537 & 50.78 \\
  & $\epsilon\approx6$  & 0.00007306 & 0.0000012885 & 54.66 \\
  & $\epsilon\approx12$ & 0.00004086 & 0.03112375   & 56.15 \\
\hline\hline
\end{tabular}
\end{table}

\pagebreak

Figure \ref{fig:gnn_adamw_eps_3} shows DP-AdamW training, validation and test losses across \([10^{-1},10^{-6}]\) for learning rate and \([10^{-1},10^{-7}]\) for weight decay using a privacy budget under \(\epsilon=3\).

\begin{figure}[!ht]
    \centering
    \includegraphics[scale=0.25]{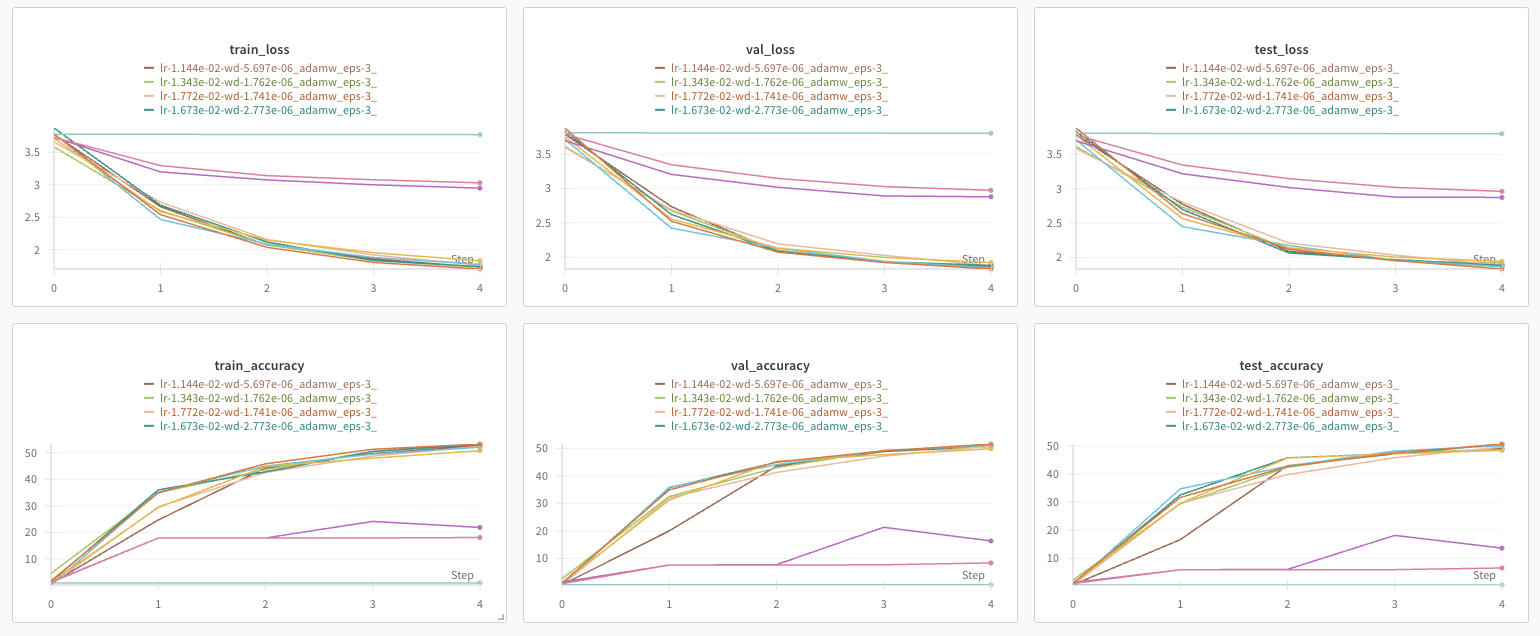}
    \caption{Losses when using DP-AdamW under \(\epsilon=3\)}
    \label{fig:gnn_adamw_eps_3}
\end{figure}

\newpage
\section{Future Work}\label{futurework}

\subsection{Broader empirical generalization within text classification}

Our text classification results in Table \ref{tab:qnli_comparison} show that DP-AdamW both with and without bias correction significantly outperforms other optimizers on the QNLI natural language task. This performance improvement is substantially higher than the corresponding improvements for the image and graph node classification tasks so we would like to investigate further. Potential patterns to examine include persistence across different text subtasks and scaling with dataset size and class balance. One potential research direction is to run our existing QNLI fine-tuning protocol on the SNLI single-sentence entailment, MultiNLI multi-genre entailment, and QQP paraphrase detection tasks described in \citep{bowman_large_2015, wang_glue_2019}, containing diverse training dataset sizes (QNLI $\approx$ 105K, SNLI $\approx$ 550K, MultiNLI $\approx$ 393K, QQP $\approx$ 364K) and subtask types.

\subsection{Scaling laws for DP-AdamW and DP-AdamW-BC}

In addition to increasing the robustness of our text classification analysis, we would also like to investigate the scaling laws of the proposed optimizers on both image and graph node classification. We hypothesize that the difference in performance on text classification compared to the other two tasks may stem from the fact that text classification operates in the large-model, large-dataset regime, while the other two tasks are in the small model regime. Note that the text classification task uses a model with over 110M parameters (bert-base-uncased), while the image classification task uses a 5-layer CNN model while the node classification task uses a 2-layer GNN model, which are several orders of magnitude smaller than that used for text. Evaluating the performance of DP-AdamW and DP-AdamW-BC on other tasks that use large models and large datasets could uncover a larger performance differential relative to other optimizers as found in the text classification results but unobserved in our image classification and node classification results.

\subsection{Examination of bias correction underperformance}

Recall that across our experiments (Table \ref{tab:cifar10_comparison}, Table \ref{tab:qnli_comparison}, Table \ref{tab:gnn_comparison}), adding bias correction consistently reduces accuracy on all three classification tasks for DP-AdamW, although the opposite phenomenon seems true for DP-Adam. We suspect that the $\Phi = (\sigma C/B)^2$ interaction term that is subtracted inside the square root in the denominator could become large and force the algorithm to clamp to $\gamma$; under decoupled weight decay, the subsequent large adaptive step is uniquely not offset inside the gradient so the next decay must account for it. We propose to train a lightweight classifier on a small dataset and plot histograms every epoch of measurements such as the clamping rate, which would help determine if the bias term becomes dominated by estimation noise.


\end{document}